\documentclass[10pt,usletter]{article}

\usepackage[totalwidth=450pt,totalheight=640pt]{geometry}

\usepackage{times}
\usepackage{helvet}
\usepackage{courier}

\usepackage{pgf}

\usepackage{amsmath}
\usepackage{amssymb}
\usepackage{amsthm}

\renewcommand{\phi}{\varphi}

\renewcommand{\emptyset}{\varnothing}  
\newcommand{\union}{\cup} 		

\newcommand{\setdiff}{-}
\newcommand{\card}[1]{{|#1|}}		
\newcommand{\set}[1]{\{{#1}\}}          
\newcommand{\st}{\ |\ }		     	

\newcommand{\tuple}[1]{\langle{#1}\rangle}  
\newcommand{\seq}[1]{\langle #1 \rangle}

\theoremstyle{definition}

\theoremstyle{plain}
\newtheorem{theorem}{Theorem}

\newtheorem{proposition}{Proposition}

\newcommand{\instance}[1]{{\mathbb{#1}}}

\newcommand{\cc}[1]{{\mbox{\textnormal{\textbf{#1}}}}}  

\newcommand{\poly}{\cc{P}}
\newcommand{\NP}{\cc{NP}}

\newcommand{\PSPACE}{\cc{\textsc{Pspace}}}

\newcommand{\FPT}{\cc{FPT}}
\newcommand{\Weft}{{\cc{W}}}
\newcommand{\W}[1]{{\Weft}{{[#1]}}}

\newcommand{\insti}{\instance{I}}
\newcommand{\iplan}{\instance{P}}      

\newcommand{\strips}{{\textsc{Strips}}}
\newcommand{\sasplus}{{SAS$^+$}}

\newcommand{\vars}{V} 
\newcommand{\dom}{D}  
\newcommand{\acts}{A}  
\newcommand{\init}{I}  
\newcommand{\goal}{G}  
\newcommand{\pre}{\mathrm{pre}}  
\newcommand{\eff}{\mathrm{eff}}  %

\newcommand{\varval}{x}

\newcommand{\proj}[2]{{#1[{#2}]}}

\newcommand{\undef}{\mathbf{u}}

\newcommand{\plan}{\omega}

\newcommand{\BPE}{\textsc{Bounded \sasplus\ Planning}}
\newcommand{\BPER}[1]{$\set{#1}$-\BPE}

\newcommand{\occs}{\mathbf{O}}
\newcommand{\orders}{\mathbf{P}}
\newcommand{\links}{\mathbf{L}}
\newcommand{\link}[3]{{#1\stackrel{#2}{\longrightarrow}{#3}}}

\newcommand{\vexpand}{\phantom{$\neg{{{l^lg}}}$}}

\newcommand{\pbox}[1]{\makebox[1em][l]{#1}}

\newcommand{\hy}{\hbox{-}\nobreak\hskip0pt}

\newcommand{\SB}{\{\,}%
\newcommand{\SM}{\;{|}\;}%
\newcommand{\SE}{\,\}}%

\newcommand{\mdef}{=}

\newcommand{\AAA}{\mathcal{A}}

\newcommand{\OOO}{\acts}

\newcommand{\SAS}{\mbox{\sasplus}}

\newcommand{\post}{\eff}

\newcommand{\undv}{\undef}

\newcommand{\ope}{\textup{Act}}
\newcommand{\varR}{\textup{Var}}
\newcommand{\domR}{\textup{Dom}}

\newcommand{\goalRV}{\textup{Goalv}}

\newcommand{\postR}{\textup{Eff}}
\newcommand{\preRV}{\textup{Prev}}
\newcommand{\postRV}{\textup{Effv}}

\newcommand{\opeNS}{\textup{Act}}
\newcommand{\varRNS}{\textup{Var}}
\newcommand{\domRNS}{\textup{Dom}}

\newcommand{\initRNS}{\textup{Init}}

\newcommand{\goalRVNS}{\textup{Goalv}}

\newcommand{\preRNS}{\textup{Pre}}
\newcommand{\postRNS}{\textup{Eff}}
\newcommand{\preRVNS}{\textup{Prev}}
\newcommand{\postRVNS}{\textup{Effv}}

\newcommand{\diffRNS}{\textup{Diff-act}}

\newcommand{\fvalue}{\textit{value}}
\newcommand{\fcheckpre}{\textit{check-pre}}
\newcommand{\fcheckpreall}{\textit{check-pre-all}}
\newcommand{\fcheckpregoal}{\textit{check-goal}}

\usepackage{natbib}

\newcommand{\shortcite}[1]{\citeyear{#1}}
\newcommand{\shortcitec}[2]{\citeyear[#1]{#2}}

\title{The Complexity of Planning Revisited -- A Parameterized Analysis}

\author{
Christer B\"{a}ckstr\"{o}m$^1$,
Yue Chen$^2$,
Peter Jonsson$^1$,
Sebastian Ordyniak$^2$, and
Stefan Szeider$^2$\\[0.1cm]
\mbox{}\small$^1$Department of Computer Science, Link{\"o}ping University,
Link{\"o}ping, Sweden\\
\small christer.backstrom@liu.se, peter.jonsson@liu.se\\
\mbox{}\small$^2$Institute of Information Systems, Vienna University of Technology,
Vienna, Austria\\
\small chen@kr.tuwien.ac.at, ordyniak@kr.tuwien.ac.at, stefan@szeider.net
}

\date{}

\begin{document}

\maketitle

\begin{abstract}
The early classifications of the computational 
complexity of planning under various restrictions
in \strips\ (Bylander)
and \sasplus\ (B\"{a}ckstr\"{o}m and Nebel)
have influenced following research in planning
in many ways.
We go back and reanalyse their
subclasses, but this time using the more modern
tool of parameterized complexity analysis.
This provides new results that
together with the old results give a more
detailed picture of the complexity landscape.
We demonstrate separation results not possible
with standard complexity theory,
which contributes to explaining  why certain cases of planning
have seemed simpler in practice than  theory
has predicted.
In particular, we show that certain restrictions of
practical interest are tractable in the parameterized
sense of the term, and that
a simple heuristic is sufficient to make a well-known 
partial-order planner exploit this fact.
\end{abstract}

\section{Introduction}

Bylander \shortcite{Bylander:aij94}
 made an extensive analysis of the computational complexity
of propositional \strips\ under various restrictions, 
like limiting
the number of preconditions or effects.
B\"{a}ckstr\"{o}m and Nebel \shortcite{Backstrom:Nebel:ci95}
 made a similar analysis of planning
with multi-valued state variables in the \sasplus\ formalism, 
investigating the complexity of all combinations of the P, U, B
and S restrictions introduced by B\"{a}ckstr\"{o}m and Klein
\shortcite{Backstrom:Klein:ci91}.
These were among the first attempts to understand why and when
planning is hard or easy and have had heavy influence on 
recent research in planning, of which we list a few
representative examples.
Gim\'{e}nez and Jonsson \shortcite{Gimenez:Jonsson:jair08},
Chen and Gim\'{e}nez \shortcite{Chen:Gimenez:jcss10}
as well as
Katz and Domshlak \shortcite{Katz:Domshlak:jair08}
have studied the complexity of planning for various 
restrictions on the causal graph, the latter also considering
combinations with restrictions P and U.
Katz and Domshlak
 further pointed out a particularly important usage of such 
results, saying:
\begin{quote}
  Computational tractability can be an invaluable tool
  even for dealing with problems that fall outside all
  the known tractable fragments of planning.
  For instance, tractable fragments of planning provide
  the foundations for most (if not all) rigorous
  heuristic estimates employed in planning as
  heuristic search.
\end{quote}
Two examples of slightly different ways to do this
are the following.
Helmert \shortcite{Helmert:aips04}  
used a planning algorithm for a simpler restricted problem
to compute heuristic values for subproblems and then
combine these values.
Similarly, the popular $h^ +$ heuristic \cite{Hoffmann:jair05}
exploits Bylander's results that planning is simpler with
only positive preconditions and uses this as a relaxation
for computing a heuristic value.
As a complement to such analyses of restricted planning
lanugages, Helmert \shortcite{Helmert:icaps06} studied 
the complexity and inherent restrictions in a number
of application problems.

We revisit these early classifications of \strips\ and of
\sasplus, but using parameterized complexity
analysis rather than standard complexity analysis.
Parameterized complexity analysis was invented to enable a more
fine-grained analysis than standard complexity analysis allows,
by treating a parameter as independent of the instance
rather than being a part of it.
Somewhat simplified, the idea is as follows.
Consider some problem and let $n$ denote the instance size.
We usually consider  a problem as tractable if it can be solved
by some algorithm in $O(n^c)$ time, that is, in polynomial time.
For many problems, like the \NP-hard problems, we do not know of any 
significantly faster way to solve them than doing brute-force search, 
which typically requires 
requires exponential, or at least super-polynomial, time in $n$.
In practice the search is often not exponential in the size
of the whole instance, but rather in some
smaller hard part of it.
In these cases the complexity may rather be something
like $O(2^k n^c)$ where $k$ is a parameter that is typically
independent of the instance size $n$.
Thus, the combinatorial explosion is confined to the parameter~$k$.
We say that a problem is \emph{fixed-parameter tractable (FPT)} 
if it can be solved in this way.
This is the essence of parameterized complexity theory and provides
a tractability concept which is more relaxed than the usual one,
while correlating better with tractability in practice for
real-world problems.
The theory also offers various classes for problems
that are not  FPT, for example $\W{1}$ and $\W{2}$.
Parameterized complexity analysis has contributed 
fundamental new insights into
complexity theory \cite{Downey:Fellows:book99}.
It is nowadays a very common
technique in many areas of computer science, including many
subareas of AI, like 
non-monotonic reasoning \cite{Gottlob:etal:aaai06},
constraints \cite{Gaspers:Szeider:ijcai11},
social choice \cite{Brandt:etal:ijcai11}
and argumentation \cite{Ordyniak:Szeider:ijcai11}.
The examples in planning are rare, however.
Downey, Fellows and Stege \shortcite{Downey:etal:dimacs99}
proved that \strips\ planning is $\W{1}$-hard and
conjectured that it  is also complete for  $\W{1}$.
We disprove this conjecture and show that \strips\ 
planning is actually $\W{2}$-complete.
There is also a result by B\"{a}ckstr\"{o}m and Jonsson
\shortcite{Backstrom:Jonsson:socs11} that \strips\ planning
is FPT under a certain restriction
that  deliberately lower-bounds the plan length, thus
not contradicting our results.
This restriction was motivated by a different agenda, 
studying the expressive
power of planning languages in general rather than
subclasses of a particular~language.

The parameterized analyses of planning that we provide in this paper
does not replace the earlier results or make them obsolete.
Since the parameterized complexity classes and the standard ones
are not comparable, our results must be viewed as supplementary,
providing further information.
If we consider the previous classifications together with
our parameterized classification we get a more detailed and informative
picture of planning complexity than by considering either of them alone.
This sheds new light on the discrepancy between
theoretical and practical results regarding the difficulty 
of planning.
For instance, while B\"{a}ckstr\"{o}m and Nebel proved that 
restriction U (actions can change only one variable)
does not make planning easier under standard analysis, 
we show that 
it is actually easier from a parameterized point of view.
This is interesting since restriction U has been considered
acceptable in some practical applications of planning, 
for instance on-board planning in spacecrafts
\cite{Williams:Nayak:ijcai97,Brafman:Domshlak:jair03}.
Furthermore, B\"{a}ckstr\"{o}m and Nebel showed that planning
is \NP-hard under restriction P (there are never two actions
that set the same variable value)
but did not provide any better
upper bound than in the unrestricted case.
We show that planning is actually  FPT 
under this restriction.
We also show that a standard partial-order planning 
algorithm \cite{McAllester:Rosenblitt:aaai91} can
exploit this fact with a minor modification that could be 
implemented as a heuristic.
This suggests that many successful applications of planning
might be cases where the problem is ``almost tractable'' and
the algorithm used happens to implicitly exploit this.
This is in line with the claim by Downey et.~al.
\shortcite{Downey:etal:cj08} that in many cases existing algorithms
with heuristics turn out to already be FPT algorithms.

The rest of the paper is laid out as follows.
Section 2  defines some concepts of parameterized
complexity theory and Section 3 defines the \sasplus\ 
and \strips\ languages.
The hardness results are collected in Section 4
and the membership results in Section 5, including
the result on using an existing planning algorithm.
Section 6 summarizes the results of the paper and
discusses some observations and consequences.
The paper ends with a discussion in Section~7.

\section{Parameterized Complexity}

We define the basic notions of Parameterized Complexity and
refer to other sources~\cite{Downey:Fellows:book99,Flum:Grohe:book06} 
for an in-depth treatment. 
A \emph{parameterized problem} is a set of pairs 
$\tuple{\insti,k}$,
the \emph{instances}, where $\insti$ is the main part and $k$ 
the \emph{parameter}. The parameter is usually a non-negative integer.
A parameterized problem is {\em fixed-parameter tractable (FPT)} if
there exists an algorithm that solves any instance $\tuple{\insti,k}$ of
size $n$ in time $f(k)n^{c}$ where $f$ is an arbitrary computable
function and $c$ is a constant independent of both $n$ and $k$. 
\FPT\ is the class of all fixed-parameter
tractable decision problems.

Parameterized complexity offers a completeness theory, similar
to the theory of NP-completeness, that allows the accumulation of
strong theoretical evidence that some parameterized problems
are not fixed-parameter tractable. This theory is based on a
hierarchy of complexity classes
\[\FPT \subseteq \W{1} \subseteq \W{2} \subseteq \W{3} \subseteq \cdots\]
where all inclusions are believed to be strict. Each class $\W{i}$ contains
all parameterized decision problems that can be reduced to a certain canonical
parameterized problem (known as {\sc Weighted $i$-Normalized Satisfiability}) 
under {\em parameterized} reductions.
A parameterized problem $L$ reduces to a parameterized problem $L'$
if there is a mapping $R$ from instances of $L$ to instances of $L'$ such
that
\begin{quote}
\begin{enumerate}
\item
$\tuple{\insti,k}$ is a {\sc Yes}-instance of $L$ if and only if $\tuple{\insti',k'}=R(\insti,k)$
is a {\sc Yes}-instance of $L'$,

\item
there is a computable function $g$ such that $k' \leq g(k)$, and

\item
there is a computable function $f$ and a constant $c$ such that $R$ can be
computed in time $O(f(k) \cdot n^c)$, where $n$ denotes the size
of $\tuple{\insti,k}$.
\end{enumerate}
\end{quote}
Not much is known about the relationship between the parameterized
complexity classes and the standard ones, except that $\poly \subseteq \FPT$.

\section{Planning Framework}\label{sec:planning-framework}

Let $\vars = \set{v_1,\ldots,v_n}$ be a finite set of
\emph{variables} over a finite \emph{domain} $\dom$.
Implicitly define $\dom^+ = \dom \union \set{\undef}$,
where  $\undef$ is a special value not present in $\dom$.
Then  $\dom^n$ is the set of \emph{total states}
and $(\dom^+)^n$ is the set of \emph{partial states}
over $\vars$ and $\dom$, where $\dom^n \subseteq (\dom^+)^n$.
The value of a variable $v$ in a state $s \in (\dom^+)^n$
is denoted $\proj{s}{v}$.
A \emph{\sasplus\ instance} is a tuple
$\iplan = \tuple{\vars,\dom,\acts,\init,\goal}$
where  $\vars$ is a set of variables, 
$\dom$ is a domain,
$\acts$ is a set of \emph{actions},
$\init \in \dom^n$ is the \emph{initial state}
and $\goal \in (\dom^+)^n$ is the \emph{goal}. 
Each action $a \in \acts$ has 
a \emph{precondition} $\pre(a) \in (\dom^+)^n$ and
an \emph{effect} $\eff(a) \in (\dom^+)^n$.
We will frequently use the convention that a variable has value $\undef$
in a precondition/effect unless a value is explicitly specified.
Let $a \in \acts$ and let $s \in \dom^n$.
Then $a$ is \emph{valid in $s$} if for all $v \in \vars$,
either $\proj{\pre(a)}{v} = \proj{s}{v}$ or $\proj{\pre(a)}{v} = \undef$.
Furthermore, the \emph{result of $a$ in $s$} is a state  $t \in \dom^n$
defined such that for all $v \in \vars$,
 $\proj{t}{v} = \proj{\eff(a)}{v}$ if $\proj{\eff(a)}{v} \neq \undef$
and $\proj{t}{v} = \proj{s}{v}$ otherwise.

Let $s_0, s_\ell \in \dom^n$ and 
let $\plan = \seq{a_1,\ldots,a_\ell}$ be a sequence of actions.
Then $\plan$ is a \emph{plan from $s_0$ to $s_\ell$} if
either\\
\indent
1) $\plan = \seq{}$ and $\ell = 0$ or\\
\indent
2) there are states $s_1,\ldots,s_{\ell-1} \in \dom^n$
such that for all $i$, where $1 \leq i \leq \ell$,
$a_i$ is valid in $s_{i-1}$ and $s_i$ is the result of $a_i$ in $s_{i-1}$.
A state $s \in \dom^n$ is a \emph{goal state}
if for all $v \in \vars$,
either $\proj{\goal}{v} = \proj{s}{v}$ or
$\proj{\goal}{v} = \undef$.
An action sequence $\plan$ is a \emph{plan for $\iplan$} if
it is a plan from $\init$ to some goal state 
$s \in \dom^n$.
We will study the following problem:

\smallskip

\begin{quote}
\noindent
\BPE\\
\textit{Instance:}
A tuple $\tuple{\iplan,k}$ where $\iplan$ is a \sasplus\ 
instance and $k$ is a positive integer.\\
\textit{Parameter:}
The integer $k$.\\
\textit{Question:}
Does $\iplan$ have a plan of length at most $k$?
\end{quote}

\smallskip

\noindent
We will consider the following four restrictions,
originally defined by B\"{a}ckstr\"{o}m and Klein
\shortcite{Backstrom:Klein:ci91}.
\begin{quote}
\begin{description}
  \item[P:]
    For each $v \in \vars$ and each $x \in \dom$
    there is at most one $a \in \acts$ such that
    $\proj{\eff(a)}{v} = x$.
  \item[U:]
    For each $a \in \acts$,  $\proj{\eff(a)}{v} \neq \undef$ 
    for exactly one $v \in \vars$.
  \item[B:]
    $\card{\dom} = 2$.
  \item[S:]
    For all $a,b \in \acts$ and all $v \in \vars$,
    if 
    $\proj{\pre(a)}{v} \neq \undef$,
    $\proj{\pre(b)}{v} \neq \undef$ and
    $\proj{\eff(a)}{v} = \proj{\eff(b)}{v} = \undef$,
    then $\proj{\pre(a)}{v} = \proj{\pre(b)}{v}$.
\end{description}
\end{quote}

\noindent
For any set $R$ of such restrictions we write \mbox{$R$-\BPE}
to denote the restriction of \BPE\ to only instances
satisfying the restrictions in $R$.

The propositional \strips\ language can be treated as 
the special case of \sasplus\ satisfying restriction B.
More precisely, this corresponds to the variant of \strips\ 
that allows negative preconditions.

\section{Hardness Results}

In this section we prove the two main hardness results of this paper.
For the first proof we need the following
\W{2}-complete problem~\cite[p.~464]{Downey:Fellows:book99}.

\smallskip
\begin{quote}
\noindent
    \textsc{Hitting Set}\\
\noindent    
    \emph{Instance:} A finite set $S$, a collection $C$ of 
    subsets of $S$ and an integer $k \leq \card{C}$.\\
\noindent    
    \emph{Parameter:} The integer $k$.\\  
\noindent
    \emph{Question:} Is there a hitting set $H \subseteq S$ 
    such that\linebreak[4]
    $\card{H} \leq k$ and $H \cap c \neq \emptyset$ 
    for every $c \in C$?
\end{quote}

\begin{theorem}\label{arb-hard}
  \BPER{B,S}\ 
  is $\W{2}$-hard, even
  when the actions have no preconditions.
\end{theorem}

\begin{proof}
  By parameterized reduction from \textsc{Hitting Set}.
  Let $\insti=\tuple{S,C,k}$ be an instance of this problem. We construct an
  instance $\insti'=\tuple{\iplan,k'}$, 
  where $\iplan=\tuple{\vars,\dom,\acts,\init,\goal}$, of the
  \BPER{B,S}\ 
  problem such that $\insti$ has a
  hitting set of size at most $k$ if and only if there is a plan of
  length at most $k'=k$ for $\insti'$ as follows. 
  Let $\vars=\SB v_c \SM c \in C \SE$ and
  let $\acts = \set{a_e \st e \in S}$ where
  $\post(a_e)[v_c]=1$ if $e \in c$ and $\post(a_e)[v_c]=\undef$
  otherwise. 
  We set $\init=\tuple{0,\ldots,0}$ and
  $\goal=\tuple{1,\ldots,1}$.
  Clearly, $\iplan$
  satisfies restrictions B and S,
  and the actions have no preconditions.  
  It is now  routine to show that $\iplan$ has a plan of
  length at most $k'$ if and only if $\insti$ has a hitting set of
  size~$k$.
\end{proof}

  We continue with the second result.
  The following problem
  is \W{1}-complete~\cite{Pietrzak:jcss03}.

  \smallskip
  \begin{quote}
  \noindent
  \textsc{Partitioned Clique}\\
  \noindent  
  \emph{Instance:} A $k$-partite graph $G=\tuple{V,E}$ with partition
  $V_1,\dots,V_k$ such that $|V_i|=|V_j|=n$ for all $i$,\linebreak[4]
  where $1\leq i<j
  \leq k$.\\
  \noindent  
  \emph{Parameter:} The integer $k$.\\  
  \noindent
  \emph{Question:} Are there nodes $v_1,\dots,v_k$ such that
  $v_i\in V_i$ for all $i$, where $1\leq i \leq k$ and,\linebreak[4]
  $\{v_i,v_j\}\in E$ for all $i$, where $1 \leq i < j \leq k$? 
  (The graph 
  $ 
  \{\{v_1,\dots,v_k\},\, \{ \{v_i,v_j\} \st
  1\leq i < j \leq k \} \}$ is a \emph{$k$-clique} of $G$.)
  \end{quote}
  
  \smallskip

\begin{theorem} \label{unary-hard}
  \BPER{U,B,S}\ 
 is $\W{1}$-hard, even for instances where every action has at most 
 one precondition and one postcondition.
\end{theorem}
\begin{proof}
  By parameterized reduction from \textsc{Partitioned Clique}.  
  Let $G = \tuple{V, E}$ be a $k$-partite
  graph where $V$ is partitioned into $V_1,\dots, V_k$. 
  Let $k_2 = \binom{k}{2}$
  and $k'=7k_2+k$.  
  We define $J_i=\SB j \SM 1 \leq j \leq k \textup{
    and }j\neq i \SE$ for every $1 \leq i \leq k$. 

  For the  \BPER{U,B,S}\ instance $\iplan$ we introduce 
  four kinds of variables:\\
  \indent
  1) For every $e\in E$ we introduce an \emph{edge variable}
    $x(e)$.\\
  \indent
  2) For every $1\leq i \leq k$ and $v\in V_i$ we introduce $k - 1$
    \emph{vertex variables} $x(v, j)$ where $j \in J_i$.\\
  \indent
  3) For every $1\leq i\leq k$ and every $j\in J_i$ we introduce a
    \emph{checking variable} $x(i, j)$.\\
  \indent
  4) For every $v\in V$, we introduce a
    \emph{clean-up variable} $x(v)$.

\smallskip

\noindent 
We also introduce five kinds of actions:

  \indent
  1) For every $e\in E$ we introduce an action $a^{e}$ such that
  $\post(a^{e})[x(e)] = 1$.\\
  \indent
  2) For every $e = \{v_i, v_j\}\in E$ where $v_i \in V_{i}$ and
  $v_j\in V_j$, we introduce two actions $a^{e}_{i}$ and $a^{e}_{j}$ such
  that $\pre(a^{e}_{i})[x(e)] = 1$, $\post(a^{e}_{i})[x(v_i, j)]=1$,
  $\pre(a^{e}_{j})[x(e)] = 1$ and $\post(a^{e}_{j})[x(v_j, i)]=1$.\\
  \indent
  3) For every $v \in V_i$ and $j\in J_i$, we introduce 
  an action $a^{v}_{j}$ such that $\pre(a^{v}_{j})[x(v,j)] = 1$ 
  and $\post(a^{v}_{j})[x(i, j)] = 1$.\\
  \indent
  4) For every $v \in V$, we introduce an
  action $a_v$ such that $\post(a_v)[x(v)] = 1$.\\
  \indent
  5) For every $v\in V_i$, for some $1\leq i \leq k$, and  $j\in J_i$,
  we introduce an action $a^j_v$ such that $\pre(a^j_v)[x(v)] = 1$ and
  $\post(a^j_v)[x(v, j)] = 0$.\\
\noindent
Let $\acts_1,\ldots,\acts_5$ be sets of actions corresponding 
to these five groups, 
and let 
$\acts = \acts_1 \union \ldots \union \acts_5$
be the set of all actions.
Let $\init=\tuple{0,\ldots,0}$ and  define $\goal$ such that
all checking variables $x(i, j)$ are $1$, all vertex variables
$x(v, j)$ are $0$ and 
the rest are $\undef$.

We now need to prove that $G$ has a $k$-clique if and
  only if there is a plan for $\iplan$ of length at most $k'$.
We sketch the leftward direction; the opposite is similar.
Assume
$G$ has a $k$-clique $K = \tuple{V_K, E_K}$
where $V_K = \{v_1, \ldots, v_k\}$ with $v_i\in V_i$ for every 
$1\leq i \leq k$.
For all $1\leq i < j \leq k$, we apply the actions
$a^{\{v_i,v_j\}}\in\OOO_1$ and
$a^{{\{v_i,v_j\}}}_{i},a^{{\{v_i,v_j\}}}_{j}\in\OOO_2$. This gives
$3k_2$ actions.
Then for each checking variable $x(i, j)$, for every $1\leq i\leq k$ and
$j\in J_i$, we apply $a^{v_i}_{j}\in\OOO_3$. This gives
$2k_2$ actions. 
Now we have all checking variables set to the required value $1$, but the
vertex variables $x(v_i, j)$, for $1\leq i\leq k$ and $j\in J_i$, still bear
the value $1$ which will have to be set back to $0$ in the goal state. So we
need some actions to ``clean up'' the values of these vertex variables.
First we set up a cleaner for each vertex $v_i$ by applying
$a_{v_i}\in\OOO_4$. This gives $k$ actions. Then we
use $a^j_{v_i}\in\OOO_5$ for all $j\in J_i$ to set the vertex variables
$x(v_i, j)$ to $0$. This requires
$2k_2$ actions. We observe that all the checking
variables are now set to $1$, and all vertex variables are set to $0$. The
goal state is therefore reached from the initial state by the execution
of exactly $k'= k + 7k_2$ actions, as required.
\end{proof}

\section{Memberhip Results}

Our membership results are based
on {\em first-order (FO) model checking} (Sec.~\ref{sec:modelcheck}) and
{\em partial-order planning} (Sec.~\ref{sec:partialplan}).

\subsection{Model Checking}
\label{sec:modelcheck}

For a class of FO formulas $\Phi$ we define the following parameterized
decision problem.

  \smallskip
  \begin{quote}
  \noindent
  \textsc{$\Phi$-FO Model Checking}\\
  \noindent
  \emph{Instance:} A finite structure $\AAA$, an FO formula $\phi \in \Phi$.\\
  \noindent
  \emph{Parameter:} The length of $\phi$.\\
  \noindent
  \emph{Question:} Does $\phi$ have a model?
  \end{quote}
  \smallskip

\noindent
Let $\Sigma_1$ be the class of all FO formulas of the form
$\exists x_1 \dots \exists x_t . \phi$
where $t$ is arbitrary and $\phi$ is a
quantifier-free FO formula.
For arbitrary positive integer $u$, let
$\Sigma_{2,u}$ denote the class of all
FO formulas of the form
$\exists x_1 \dots \exists x_t \forall y_1 \dots \forall y_u . \phi$
where $t$ is arbitrary and $\phi$ is a quantifier-free FO formula.
Flum and Grohe~\shortcitec{Theorem 7.22}{Flum:Grohe:book06} have shown
the following result.

\begin{proposition}\label{pro:fo-model-w1} \label{pro:fo-model-w2}
  The problem \textsc{$\Sigma_1$-FO Model Checking} is $\W{1}$-complete.
  For every positive integer $u$ the problem \textsc{$\Sigma_{2,u}$-FO
    Model Checking} is $\W{2}$-complete.
\end{proposition}

We will reduce planning to model checking,
so for an arbitrary planning instance ${\insti}=\tuple{\iplan,k}$
(where $\iplan=\tuple{\vars,\dom,\acts,\init,\goal}$) of the
\BPE\ problem we need a relational structure $\AAA(\iplan)$ 
defined as:

\begin{itemize}
\item The universe of $\AAA(\iplan)$ is $\vars \cup \acts \cup \dom^+$.
\item $\AAA(\iplan)$ contains the unary relations
$\varRNS \mdef \vars$, $\opeNS \mdef \acts$, and $\domRNS \mdef \dom^+$
together with the following relations of higher arity:
  \begin{itemize}
  \item $\initRNS \mdef \SB \tuple{v,\varval} \in \vars
    \times \dom \SM I[v]=\varval \SE$,
  \item $\goalRVNS \mdef \SB \tuple{v,\varval} \in \vars
    \times \dom \SM \goal[v]=\varval \neq \undv \SE$,
  \item $\preRNS \mdef \SB \tuple{a,v} \in \acts \times \vars
    \SM \pre(a)[v] \neq \undv 
    \SE$,
  \item $\postRNS \mdef \SB \tuple{a,v} \in \acts \times \vars
    \SM \post(a)[v] \neq \undv \SE$,
  \item $\preRVNS \mdef 
    \SB \tuple{a,v,\varval} \in \acts\times \vars \times \dom \SM 
      \pre(a)[v] = \varval \neq \undv \SE $
  \item $\postRVNS \mdef \SB \tuple{a,v,\varval} \in \acts
    \times \vars \times \dom \SM 
    \post(a)[v]=\varval \neq \undv \SE$.
  \end{itemize}
\end{itemize}

\begin{theorem}\label{in-W2}
  \BPE\ 
  is in $\W{2}$.
\end{theorem}
\begin{proof}
  By parameterized reduction to the $\W{2}$-complete 
  problem \textsc{$\Sigma_{2,2}$-FO
  Model Checking}. Let
  ${\insti}=\tuple{\iplan,k}$ (where $\iplan=\tuple{\vars,\dom,\acts,\init,\goal}$) be an 
  instance of 
  \textsc{Bounded $\SAS$ Planning}. We construct an
  instance ${\insti}'=\tuple{\AAA(\iplan),\phi}$ of \textsc{$\Sigma_{2,2}$-FO Model
    Checking}  such that ${\insti}$ has a solution if and
  only if ${\insti}'$ has a solution and the size of the formula $\phi$ is
  bounded by some function that only depends on $k$.
  Assume without loss of generality that $A$ contains a dummy action $\hat{a}$
  with no preconditions and no effects.
  To define $\phi$ we first need the following definitions.

   We  define a formula $\fvalue(\langle a_1,\ldots,a_i \rangle,v,\varval)$
  such that $\fvalue(\langle \rangle,v,\varval)=\initRNS(v, \varval)$ and
$\fvalue(\langle a_1,\ldots,a_i \rangle,v,\varval)  = 
    (\fvalue(\langle a_1,\ldots,a_{i-1} \rangle,v,\varval) \land \lnot 
   \postR(a_i,v) )
    \lor 
     \postRV(a_i,v,\varval)$ for every $0 \leq i \leq k$,
     which holds if  applying $a_1,\ldots,a_i$ in
     state $\init$
results in a state $s$ such that $s[v]=\varval$.

  We also define  a formula
    $\fcheckpre(\langle a_1,\ldots,a_i \rangle,v,\varval) = \preRV(a_i,v,\varval)
    \rightarrow \fvalue(\langle a_1,\ldots,a_{i-1} \rangle,v,\varval)$
for all $1 \leq i \leq k$,
  that is, $\forall v \forall \varval . \varR(v) \wedge \domR(\varval) \wedge 
  \fcheckpre(\langle a_1,\ldots,a_i \rangle,v,\varval)$ holds if
  all preconditions 
  of  action $a_i$ are satisfied after 
  actions $a_1,\dots,a_{i-1}$ have been executed in  state $\init$.
  We similarly define a formula
  $\fcheckpreall(\langle a_1,\ldots,a_k \rangle,v,\varval) =
  \bigwedge_{i=1}^k\fcheckpre(\langle a_1,\ldots,a_i \rangle,v,\varval)$
  that ``checks'' the preconditions of all actions in a sequence.

  Finally, define 
  $\fcheckpregoal(\langle a_1,\ldots,a_k \rangle,v,\varval) = \goalRV(v,\varval)
  \rightarrow \fvalue(\langle a_1,\ldots,a_k \rangle,v,\varval)$. The formula
  $\forall v \forall \varval . \varR(v) \wedge \domR(\varval) \wedge 
\fcheckpregoal(\langle a_1,\ldots,a_k \rangle,v,\varval)$
  holds if the
  goal state is reached after the execution of the sequence
  $a_1,\ldots,a_k$ in the state $\init$.

  We can now define the formula $\phi$ itself as:
\[\begin{array}{rcl}
\phi & = & \exists a_1 \ldots \exists a_k \forall v \forall \varval\,.\,\\
& & \phantom{m}(\bigwedge_{i=1}^k\ope(a_i))\ \land\\
& & \phantom{m}(\varR(v) \land \domR(\varval) 
    \: \rightarrow \\
& & \phantom{mm}\fcheckpreall(\langle a_1,\dots,a_k \rangle,v,\varval) \: \land \\
& & \phantom{mm}\fcheckpregoal(\langle a_1,\dots,a_k \rangle, v,\varval)).
\end{array}\]
  Evidently $\phi \in \Sigma_{2,2}$, the length of $\phi$ is bounded
  by some function that only depends on $k$ and $\AAA(\iplan) \models
  \phi$ if and only if $\iplan$ has a plan of length at most $k$.
  The dummy action guarantees that there is a plan exactly of length
  $k$ if there is a shorter plan.
\end{proof}

The proof of the next theorem resembles the previous proof but
the details are a bit involved. Thus, we only provide a high-level
description of it.

\begin{theorem}\label{in-W1}
  \BPER{U}\ 
  is in $\W{1}$.
\end{theorem}
\begin{proof}[Proof sketch:]
In order to show $\W{1}$\hy membership of 
\BPER{U}\ 
we will reduce this problem to \textsc{$\Sigma_1$\hy FO
  Model Checking} and the basic idea is fairly close to the
proof of Theorem~\ref{in-W2}. However, we cannot directly 
express within $\Sigma_1$
that all the preconditions of an action are satisfied, as this would
require a further universal quantification and thus move the formula to
$\Sigma_{2,u}$. 
Hence, we avoid the universal quantification with a trick:  
we observe that the preconditions only need to
be checked with respect to at most $k$ ``important'' variables,
that is, the
variables in which the preconditions of an action differ from the
initial state. If the precondition differs in more than~$k$ variables from
the initial state, then it cannot be used in any plan of length $k$.  
It is now possible to guess the important variables with existential
quantifiers. 

It remains  to check that  all the
significant variables are among these guessed variables. 
We  do this without universal quantification by adding dummy
elements $d_1,\dots,d_k$ and a relation $\diffRNS$ to the relational
structure $\AAA(\iplan)$.  The relation associates with each action
exactly $k$ different elements. These elements consist of all the
important variables of the action, say the number of these variables is
$k'$, plus $k-k'$ dummy elements. Hence, by guessing these $k$ elements
and eliminating the dummy elements, the formula knows all the
significant variables of the action and can check the preconditions
without a universal quantification.
\end{proof}

\subsection{Partial-order Planning}
\label{sec:partialplan}

To prove that \BPER{P} is in \FPT\ 
we  use a slight modification of  the well-known 
planning algorithm by McAllester and Rosenblitt
\shortcite{McAllester:Rosenblitt:aaai91},
which we refer to as MAR.
It appears in  Figure~\ref{alg:mar},
combining the original  and the modified versions
into one.
The only modification is the value of $L'$, which could
easily be implemented as a heuristic for the original algorithm.
The algorithm is  generalized to \sasplus\ rather
than propositional \strips, which is  straightforward
and appears in the literature
\cite{Backstrom:aaai94ws}.
We only explain the algorithm and our notation,
referring the reader to the original paper for details.

The algorithm works on a partially ordered set of
\emph{action occurences}, each occurence being a unique 
copy of an action.
For each precondition $\proj{\pre(o_c)}{v} \neq \undef$ of an occurence
$o_c$, the algorithm uses a \emph{causal link} $\link{o_p}{v=x}{o_c}$
to explicitly keep track of which other occurence
$o_p$ with $\proj{\eff(o_p)}{v}=x$
guarantees this precondition.
An occurence $o_t$ is a \emph{threat} to 
$\link{o_p}{v=x}{o_c}$
if $\proj{\eff(o_t)}{v} \neq \undef$ and 
$o_p \neq o_t \neq o_c$.
A \emph{plan structure} for a planning instance
$\iplan = \tuple{\vars,\dom,\acts,\init,\goal}$
is a tuple $\Theta = \tuple{\occs,\orders,\links}$ 
where $\occs$ is a finite set of action occurences 
over $\acts$,
$\orders$ is a binary relation over $\occs$ and
$\links$ is a set of causal links.
We write $o \prec o'$ for $\tuple{o,o'} \in \orders$.
Furthermore, $\occs$ always contains the two special 
elements $o_\init,o_\goal$, where
$\eff(o_\init) = \init$, $\pre(o_\goal)=\goal$ and
$o_\init \prec o_\goal \in \orders$.
An \emph{open goal} in $\Theta$ is a tuple 
$\tuple{o,v,x}$ such that $o \in \occs$,
$\proj{\pre(o)}{v} = x \neq \undef$ and there is no
$o' \in \occs$ such that $\link{o'}{v=x}{o} \in \links$.

We say $\Theta$ is \emph{complete} if both the following
conditions hold: 
1) For all $o_c \in \occs$ and all $v \in \vars$ such that
$\proj{\pre(o_c)}{v} = x \neq \undef$, 
there is a causal link $\link{o_p}{v=x}{o_c} \in \links$. 
2) For every  $\link{o_p}{v=x}{o_c} \in \links$ and 
every threat $o_t \in \occs$ to 
$\link{o_p}{v=x}{o_c}$, either 
$o_t \prec o_p \in \orders$ or $o_c \prec o_t \in \orders$.
McAllester and Rosenblitt proved that if starting with
$\Theta= \tuple{\set{o_\init,o_\goal},
  \set{o_\init \prec o_\goal},\emptyset}$
then the algorithm fails if there is no plan and otherwise
returns a plan structure $\tuple{\occs,\orders,\links}$
such that any topological sorting of 
$\occs \setdiff \set{o_\init,o_\goal}$ consistent with
$\orders$ is a plan.

\begin{figure}
  \begin{center}
    \begin{tabular}{l}
    \pbox{1}\hspace{0em}\textbf{function} 
    Plan($\Theta = \tuple{\occs,\orders,\links}$,$k$)\\
    \pbox{2}\hspace{1em}\textbf{if} $\tuple{\occs,\orders}$ is not acyclic
    \textbf{or} $\card{\occs} > k + 2$ 
    \textbf{then} \textbf{fail}\\
    \pbox{3}\hspace{1em}\textbf{elsif} $\Theta$ is complete 
    \textbf{then} \textbf{return} $\Theta$\\
    \pbox{4}\hspace{1em}\textbf{elsif} there is an 
    $\link{o_p}{v=x}{o_c} \in \links$ with 
    a threat $o_t \in \occs$\\
    \pbox{}\hspace{2em}
    and neither $o_t \prec o_p \in \orders$ nor
    $o_c \prec o_t \in \orders$ \textbf{then} \\
    \pbox{5a}\hspace{2em}\textbf{choose} either of\\
    \pbox{5b}\hspace{3em} \textbf{return} 
    Plan($\tuple{\occs,\orders \union \set{o_t \prec o_p},\links}$,$k$)\\
    \pbox{5c}\hspace{3em} \textbf{return} 
    Plan($\tuple{\occs,\orders \union \set{o_c \prec o_t},\links}$,$k$)\\
    \pbox{6}\hspace{1em}\textbf{else} arbitrarily choose an open goal 
    $g=\tuple{o_c,v,x}$\\
    \pbox{7a}\hspace{2em}\textbf{nondeterministically do either}\\
    \pbox{7b}\hspace{2em} 1) \textbf{nondeterministically choose} an $o_p \in \occs$\\
    \pbox{  }\hspace{3em} such that $\proj{\eff(o_p)}{v}=x$\\
    \pbox{7c}\hspace{2em} 2) \textbf{if} there is an $a \in \acts$
    such that $\proj{\eff(a)}{v}=x$\\
    \pbox{  }\hspace{3em}  \textbf{then} let $o_p$ be a new occurence of $a$\\
    \pbox{8}\hspace{2em}\textbf{if} original algorithm \textbf{then}\\
    \pbox{}\hspace{3em}  $\links' :=  \set{\link{o_p}{v=x}{o_c}}$\\
    \pbox{9}\hspace{2em}\textbf{if} modifed algorithm \textbf{then}\\
    \pbox{}\hspace{3em}  $\links' := 
    \{ \link{o_p}{w=y}{o_c} \st 
    \proj{\eff(o_p)}{w} = \proj{\pre(o_c)}{w} = y$,\\
    \pbox{}\hspace{3em} \phantom{$\links' := \{$}
    $y \neq \undef$ and $\tuple{o_c,w,y}$ is an open goal $\}$\\
    \pbox{10}\hspace{2em} \textbf{return} 
    Plan($\tuple{\occs \union \set{o_p},\orders \union \set{o_p\prec o_c},
      \links \union \links' }$,$k$)\\
    \end{tabular}
    \caption{The MAR algorithm.}
    \label{alg:mar}
  \end{center}
\end{figure}

That the modified variant of MAR is correct for \sasplus-P
instances is based on the following observation about 
the original
variant applied to such instances.
Consider three occurences $o_1,o_2,o_3$ such that
$o_1$ has  preconditions $v=x$ and $w=y$
which are both effects of $o_2$.
If $v=x$ is also an effect of $o_3$, then also $w=y$ must be
an effect of $o_3$ due to restriction P.
However, the algorithm must link both conditions from the same
occurence, either $o_2$ or $o_3$, since it would
otherwise add both $o_2 \prec o_3$ and $o_3 \prec o_2$,
causing it to fail.
The set of possible outcomes for the two variants are thus identical,
but the modified variant is an FPT algorithm.

\begin{theorem}\label{th:fpt}
  \BPER{P} is in \FPT.
\end{theorem}

\begin{proof}
  Consider the modified version of MAR.
  All nodes in the search tree
  run in polynomial time in the instance size.
  The search tree contains two types of nodes:
  leafs that terminate in either line 2 or 3 and
  nodes that make a nondeterministic choice either
  in line 5 or in line 7 and then make a recursive call.
  The latter nodes correspond to branching points in
  the search tree, and we analyse their contribution to
  the search-tree size separately.

  Each time line 5 is visited, it adds a new element to 
  $\orders$,
  which thus grows monotonically along every branch
  in the search tree.
  We can thus visit line 5 at most  $(k+2)^2$ times along any
  branch since $\card{\orders} \leq (k+2)^2$. 
  There are two choices in line 5 so it  contributes
  at most a factor $2^{(k+2)^2}$ to the size of the search tree.
  Also  $\occs$ grows monotonically along every branch
  and $\card{\occs} \leq k+2$.
  At any visit to lines 6--10 there are 
  thus at most $k+1$ occurences with open goals
  and at most $k+1$ different occurences to link these goals to.
  That is, the preconditions of each occurence are partitioned
  into at most $k+1$ parts, each part having all its
  elements linked at once in line 9.
  Lines 6--10 can thus be visited at most 
  $(k+1)^2$ times along any branch in the search tree.
  Since there are at most $k+1$ existing occurences to link to
  and at most one action to instantiate as a new occurence,
  the branching factor is $k+2$.
  The contribution of this  to the size of the
  search tree is thus at most a factor $(k+2)^{(k+1)^2}$.
  Hence, the total search-tree size is at most 
  $2 \cdot 2^{(k+2)^2} (k+2)^{(k+1)^2} $
  where the factor $2$ accounts for the leaves.
  This does not depend on the instance size
  and each node is polynomial-time in the instance size
  so the modified MAR algorithm is an FPT algorithm.
\end{proof}

\section{Summary of Results}
\label{sec:class}

The complexity results for the various combinations of
restrictions P, U, B and S are displayed in
Figure~\ref{fig:pubs-lattice}.
Solid lines denote separation results by
B\"{a}ckstr\"{o}m and Nebel
\shortcite{Backstrom:Nebel:ci95},
using standard complexity analysis,
while dashed lines denote separation results 
from our parameterized analysis.
The $\W{2}$-completeness results follow from
Theorems~\ref{arb-hard} and \ref{in-W2},
the $\W{1}$-completeness results follow from
Theorems~\ref{unary-hard} and \ref{in-W1},
and the \FPT\ results follow from
Theorem~\ref{th:fpt}.

\begin{figure}[tb]
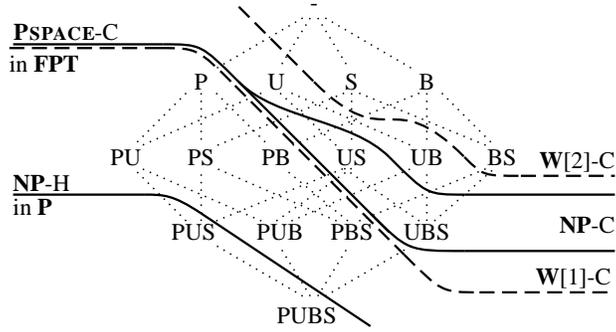

  \begin{center}
    \begin{pgfpicture}{-2cm}{-2cm}{2cm}{2cm}
      \newdimen\pubsdim
      \pgfextractx{\pubsdim}{\pgfpoint{5mm}{0mm}}
      \pgfsetxvec{\pgfpoint{5mm}{0mm}}
      \pgfsetyvec{\pgfpoint{0mm}{5mm}}

      \small
      \pgfsetlinewidth{0.2mm}
      \pgfsetdash{{0.2mm}{0.8mm}}{0mm}

      \pgfputat{\pgfxy(0,4)}{\pgfbox[center,center]{-}}

      \pgfline{\pgfxy(-0.3,3.7)}{\pgfxy(-2.7,2.3)} 
      \pgfline{\pgfxy(-0.1,3.7)}{\pgfxy(-0.9,2.4)} 
      \pgfline{\pgfxy(0.1,3.7)}{\pgfxy(0.9,2.4)}   
      \pgfline{\pgfxy(0.3,3.7)}{\pgfxy(2.7,2.3)}   

      \pgfputat{\pgfxy(-3,2)}{\pgfbox[center,center]{P}}
      \pgfputat{\pgfxy(-1,2)}{\pgfbox[center,center]{U}}
      \pgfputat{\pgfxy(1,2)}{\pgfbox[center,center]{S}}
      \pgfputat{\pgfxy(3,2)}{\pgfbox[center,center]{B}}

      \pgfline{\pgfxy(-3.4,1.7)}{\pgfxy(-4.8,0.3)} 
      \pgfline{\pgfxy(-3.0,1.7)}{\pgfxy(-3.0,0.3)} 
      \pgfline{\pgfxy(-2.6,1.7)}{\pgfxy(-1.2,0.3)} 

      \pgfline{\pgfxy(-1.4,1.7)}{\pgfxy(-4.6,0.3)} 
      \pgfline{\pgfxy(-0.9,1.7)}{\pgfxy(0.6,0.3)}  
      \pgfline{\pgfxy(-0.6,1.7)}{\pgfxy(2.6,0.3)}  

      \pgfline{\pgfxy(0.8,1.7)}{\pgfxy(-2.6,0.3)}  
      \pgfline{\pgfxy(1.0,1.7)}{\pgfxy(1.0,0.3)}   
      \pgfline{\pgfxy(1.2,1.7)}{\pgfxy(4.6,0.3)}   

      \pgfline{\pgfxy(2.8,1.7)}{\pgfxy(-0.6,0.3)}  
      \pgfline{\pgfxy(3.0,1.7)}{\pgfxy(3.0,0.3)}   
      \pgfline{\pgfxy(3.4,1.7)}{\pgfxy(4.8,0.3)}   

      \pgfputat{\pgfxy(-5,0)}{\pgfbox[center,center]{PU}}
      \pgfputat{\pgfxy(-3,0)}{\pgfbox[center,center]{PS}}
      \pgfputat{\pgfxy(-1,0)}{\pgfbox[center,center]{PB}}
      \pgfputat{\pgfxy(1,0)}{\pgfbox[center,center]{US}}
      \pgfputat{\pgfxy(3,0)}{\pgfbox[center,center]{UB}}
      \pgfputat{\pgfxy(5,0)}{\pgfbox[center,center]{BS}}

      \pgfline{\pgfxy(-3.4,-1.7)}{\pgfxy(-4.8,-0.3)} 
      \pgfline{\pgfxy(-3.0,-1.7)}{\pgfxy(-3.0,-0.3)} 
      \pgfline{\pgfxy(-2.6,-1.7)}{\pgfxy(0.8,-0.3)} 

      \pgfline{\pgfxy(-1.4,-1.7)}{\pgfxy(-4.6,-0.3)} 
      \pgfline{\pgfxy(-1.0,-1.7)}{\pgfxy(1.0,-0.3)}  
      \pgfline{\pgfxy(-0.6,-1.7)}{\pgfxy(2.6,-0.3)}  

      \pgfline{\pgfxy(0.6,-1.7)}{\pgfxy(-2.6,-0.3)}  
      \pgfline{\pgfxy(0.9,-1.7)}{\pgfxy(-0.8,-0.3)}  
      \pgfline{\pgfxy(1.2,-1.7)}{\pgfxy(4.6,-0.3)}   

      \pgfline{\pgfxy(2.8,-1.7)}{\pgfxy(0.8,-0.3)}  
      \pgfline{\pgfxy(3.0,-1.7)}{\pgfxy(3.0,-0.3)}  
      \pgfline{\pgfxy(3.2,-1.7)}{\pgfxy(4.8,-0.3)}  

      \pgfputat{\pgfxy(-3.2,-2)}{\pgfbox[center,center]{PUS}}
      \pgfputat{\pgfxy(-0.9,-2)}{\pgfbox[center,center]{PUB}}
      \pgfputat{\pgfxy(1,-2)}{\pgfbox[center,center]{PBS}}
      \pgfputat{\pgfxy(3,-2)}{\pgfbox[center,center]{UBS}}

      \pgfline{\pgfxy(-0.3,-3.7)}{\pgfxy(-2.7,-2.3)} 
      \pgfline{\pgfxy(-0.1,-3.7)}{\pgfxy(-0.9,-2.4)} 
      \pgfline{\pgfxy(0.1,-3.7)}{\pgfxy(0.9,-2.4)}   
      \pgfline{\pgfxy(0.3,-3.7)}{\pgfxy(2.7,-2.3)}   

      \pgfputat{\pgfxy(-0.2,-4.2)}{\pgfbox[center,center]{PUBS}}


      \pgfsetlinewidth{0.3mm}
      \pgfsetdash{}{0mm}

      \pgfxyline(-8.0,-1.0)(-4.5,-1.0)
      \pgfxycurve(-4.5,-1.0)(-4.0,-1.0)(-3.75,-1.0)(-3.0,-1.5)
        \pgfxyline(-3.0,-1.5)(1.5,-4.5)
        \pgfstroke
        \pgfputat{\pgfxy(-8.0,-1.1)}{\pgfbox[left,top]{in \poly}}

        \pgfxyline(-8.0,3.0)(-4.0,3.0)
        \pgfxycurve(-4.0,3.0)(-3.0,3.0)(-3.0,3.0)(-2.0,2.0)
        \pgfxyline(-2.0,2.0)(2.0,-2.0)
        \pgfxycurve(1.5,-1.5)(2.5,-2.5)(2.5,-2.5)(4.0,-2.5)
        \pgfxyline(4.0,-2.5)(8.0,-2.5)
        \pgfputat{\pgfxy(-8.0,-0.9)}{\pgfbox[left,bottom]{\NP-H}}
        \pgfputat{\pgfxy(8.0,-1.8)}{\pgfbox[right,center]{\NP-C}}
        \pgfxycurve(-2.0,2.0)(-1.0,1.0)(1.0,1.0)(2.0,0.0)
        \pgfxycurve(2.0,0.0)(3.0,-1.0)(3.0,-1.0)(4.0,-1.0)
        \pgfxyline(4.0,-1.0)(8.0,-1.0)
        \pgfstroke
        \pgfputat{\pgfxy(-8.0,3.1)}{\pgfbox[left,bottom]{\PSPACE-C}}
        \pgfsetlinewidth{0.3mm}
        \pgfsetdash{{2mm}{1mm}}{0mm}

        \begin{pgftranslate}{\pgfxy(-0.1,-0.1)}
        \pgfxyline(-8.0,3.0)(-4.0,3.0)
        \pgfxycurve(-4.0,3.0)(-3.0,3.0)(-3.0,3.0)(-2.0,2.0)
        \pgfxyline(-2.0,2.0)(2.5,-2.5)
        \pgfxycurve(2.5,-2.5)(3.5,-3.5)(3.5,-3.5)(4.5,-3.5)
        \pgfxyline(4.5,-3.5)(8.0,-3.5)
        \pgfstroke
        \pgfputat{\pgfxy(-8.0,2.8)}{\pgfbox[left,top]{in \FPT}}
        \pgfputat{\pgfxy(8.0,-3.4)}{\pgfbox[right,bottom]{\W{1}-C}}
        \end{pgftranslate}

        \begin{pgftranslate}{\pgfxy(0.0,0.0)}
        \pgfxyline(-2.0,4.0)(0.0,2.0)
        \pgfxycurve(0.0,2.0)(2.0,0.0)(2.0,2.0)(4.0,0.0)
        \pgfxycurve(4.0,0.0)(4.5,-0.5)(4.5,-0.5)(5.5,-0.5)
        \pgfxyline(5.5,-0.5)(8.0,-0.5)
        \pgfstroke
        \pgfputat{\pgfxy(8.0,-0.4)}{\pgfbox[right,bottom]{\W{2}-C}}
        \end{pgftranslate}

    \end{pgfpicture}
  \end{center}
  \caption{Complexity of \BPE\ for all combinations of
    restrictions P, U, B and S.}

  \label{fig:pubs-lattice}
\end{figure}

Bylander~\shortcite{Bylander:aij94} studied the complexity
of \strips\ under varying numbers of preconditions and
effects,
which is natural to view as a relaxation of restriction U
in \sasplus.
Table~\ref{tbl:bylander} shows such results 
(for arbitrary domain sizes $\geq 2$)
under both parameterized and standard analysis.
The parameterized results 
are derived as follows.
For actions with an arbitrary number of effects, 
the results follow from Theorems~\ref{arb-hard} and~\ref{in-W2}.
For actions with at most one effect, 
we have two cases: 
With no preconditions
the problem is trivially in \poly. Otherwise, the results
follow from Theorems~\ref{unary-hard} and~\ref{in-W1}.

We are left with the case when the number of effects is
bounded by some constant $m_e > 1$. 
B\"{a}ckstr\"{o}m~\shortcitec{proof of Theorem~6.7}{Backstrom:PhD} 
presented a polynomial
time reduction of this class of \sasplus\ instances to the class
of instances with one effect. It is easy to verify that
his reduction is a
parameterized reduction
so we have membership in \W{1} by Theorem~\ref{in-W1}. 
When $m_p \geq 1$, then we also have \W{1}-hardness by Theorem~\ref{unary-hard}.
For the final case ($m_p=0$), we
have no corresponding parameterized hardness result. 

All non-parameterized hardness results in Table~\ref{tbl:bylander}
follow directly from 
Bylander's~\shortcitec{Fig.~1 and~2}{Bylander:aij94}
complexity results for \strips. Note that we use
results both for bounded and unbounded plan existence,
 which is
justified since the unbounded case is (trivially) polynomial-time
reducible to the bounded case.
The membership results for \PSPACE\ are immediate since
\textsc{Bounded $\SAS$ Planning} is in \PSPACE.
The membership results for \NP\ (when $m_p=0$)
follow from Bylander's~\shortcite{Bylander:aij94}
Theorem~3.9, which says that every solvable \strips\ instance 
with $m_p=0$ has a plan
of length $\leq m$ where $m$ is the number of actions. 
It is easy to verify that the same bound holds for \sasplus\
instances.

\begin{table}[tb]
  \begin{center}
    
    \begin{tabular}{|l|l|l|l|} \hline    
      & $m_e=1$   & fix $m_e > 1$ & arb. $m_e$ \\ \hline
      $m_p=0$        & in \poly   & in \W{1}  & \W{2}-C \vexpand\\ 
      & in \poly      &\NP-C             & \NP-C \\ 
      \hline
      $m_p=1$        & \W{1}-C & \W{1}-C   & \W{2}-C \vexpand\\
      & \NP-H   &  \NP-H    & \PSPACE-C \\ 
      \hline
      fix $m_p > 1$  & \W{1}-C & \W{1}-C   & \W{2}-C \vexpand\\
      & \NP-H   & \PSPACE-C & \PSPACE-C \\ 
      \hline
      arb. $m_p$     & \W{1}-C & \W{1}-C   & \W{2}-C \vexpand\\ 
      & \PSPACE-C & \PSPACE-C   & \PSPACE-C \\ \hline
    \end{tabular}
    \caption{Complexity of \BPE,
      restricting the number of preconditions ($m_p$) and 
      effects ($m_e$)}
    \label{tbl:bylander}
  \end{center}
\end{table}

Since \W{1} and \W{2} are not directly comparable
to the standard complexity classes we get interesting
separations from combining the two methods.
For instance, we can single out restriction U as making planning
easier than in the general case, which is not possible with standard
analysis.
Since restrictions B and S remain as hard as the general case even
under parameterized analysis, this shows that U is a more interesting
and important restriction than the other two.
Even more interesting is that planning is in \FPT\ under restriction P,
making it easier than the combination restriction US, 
while it seems to be rather the other way around for 
standard analysis where restriction P is only known
to be hard for \NP.
In general, we see that there are still a number of open 
problems of this type in both Figure~\ref{fig:pubs-lattice}
and Table~\ref{tbl:bylander} for the standard analysis,
while there is only one single open problem for 
the parameterized analysis: hardness for the case where
$m_p=0$ and $m_e$ is fixed.

\section{Discussion}

This work opens up several new research directions. We briefly
discuss some of them below.

Although a modification was needed to make MAR an
FPT algorithm for restriction P, no modification is necessary
if also the number of preconditions of each action is bounded by a constant $c$.
Then we can even relax P, such that for some constant $d$ there can be at
most $d$ actions with the same effect.
The proof is similar to the one for Theorem~\ref{th:fpt},
using that the total number of causal links
is bounded by $c(k+1)$ and the branching factor in line 7 is $k+1+d$.
This is an important observation since many application and example
problems in planning satisfy these constraints, for instance,
many variants of the \textsc{Logistics} domain used in
the international planning competitions.
Since planners like \textsc{Nonlin} and \textsc{Snlp} are practical
variants of MAR, this may help to explain the gap between empirical
and theoretical results for many applications.

The use of parameterized analysis in planning is by no means restricted
to using plan length as parameter.
We did so only to get results that are as comparable as possible with
the previous results.
For instance,  Downey et.~al. \shortcite{Downey:etal:dimacs99}
show that \strips\ planning can be recast as
the \textsc{Signed Digraph Pebbling} problem which is modelled
as a special type of graph.
They analyse the parameterized complexity of this problem
considering also the treewidth of the graph as a parameter.
As another example, 
Chen and Gim\'{e}nez \shortcite{Chen:Gimenez:jcss10}
show that planning is in \poly\ if the size of 
the connected components in the causal graph is bounded
by a constant, and otherwise unlikely to be in \poly.
It seems natural to study this also from a parameterized
point of view, using the component size as the parameter.
It should also be noted that the parameter need not be
a single value; it can itself be a combination of
two or more other parameters.\pagebreak[4]

There are close ties between model checking and planning
and this connection deserves further study.
For instance, model-checking traces can be viewed as plans
and vice versa~\cite{Edelkamp:etal:Dagstuhl2007-944},
and methods and results have been transferred between 
the two areas in both directions 
\cite{Edelkamp:jair03,Wehrle:Helmert:sas09,Edelkamp:etal:mochart10}.
Our reductions from planning to model-checking suggest that 
the problems are related also on a more fundamental level than
just straightforward syntactical translations.

\section*{Acknowledgements}

We would like to thank the anonymous reviewers for constructive
comments.
Chen, Ordyniak, and Szeider acknowledge the support from 
the European Research Council (ERC), project COMPLEX REASON, 239962.



\end{document}